\documentclass[10pt]{article}

\usepackage[nonatbib, final]{nips_2017}

\usepackage[utf8]{inputenc} 
\usepackage{hyperref}       
\usepackage{url}            
\usepackage{booktabs}       
\usepackage{amsfonts}       
\usepackage{nicefrac}       
\usepackage{microtype}      
\usepackage{graphicx}

\PassOptionsToPackage{sort&compress}{natbib}
\usepackage{amsmath}
\usepackage{amsthm}
\usepackage{amssymb}
\usepackage{xcolor}
\usepackage{mathrsfs}
\usepackage{cleveref}
\usepackage{epstopdf}

\newcommand{\ordo}{\mathcal{O}}

\newtheorem{theorem}{Theorem}

\newtheorem{lemma}[theorem]{Lemma}
\newtheorem{proposition}[theorem]{Proposition}

\theoremstyle{definition}
\newtheorem{definition}[theorem]{Definition}
\newtheorem{example}[theorem]{Example}

\newcommand{\tr}{{\rm trace}}
\newcommand{\diag}{{\rm diag}}
\newcommand{\mR}{{\mathbb R}}
\newcommand{\ett}{{\bf 1}}

\newcommand{\Real}{\ensuremath{\mathbb{R}}}

\newcommand{\RecSpace}{\ensuremath{X}}
\newcommand{\DataSpace}{\ensuremath{Y}}

\DeclareMathOperator{\ForwardOp}{\ensuremath{\mathcal{T}}}
\DeclareMathOperator{\ForwardOpInvP}{\ensuremath{\ForwardOp^{\dagger}}}
\DeclareMathOperator{\LogLikelihood}{\ensuremath{\mathcal{L}}}
\newcommand{\loss}{L}

\newcommand{\Regularizer}{S}

\newcommand{\signal}{\ensuremath{f}}
\newcommand{\truesignal}{\signal_{\text{true}}}

\newcommand{\data}{\ensuremath{g}}
\newcommand{\noise}{\delta\data}
\newcommand{\domain}{\Omega}

\newcommand{\vparam}{\ensuremath{\boldsymbol{\Theta}}}

\newcommand{\marginalone}{\ensuremath{\mu_0}}
\newcommand{\marginaltwo}{\ensuremath{\mu_1}}
\newcommand{\iterateone}{\ensuremath{u}}
\newcommand{\iteratetwo}{\ensuremath{v}}

\usepackage{subcaption}

\newcommand{\ProdSpace}{\ensuremath{U}}

\newcommand{\learnedInv}[1]{#1^{\dagger}}

\DeclareMathOperator{\ForwardOpInvLearned}{\ensuremath{\learnedInv{\ForwardOp}_{\vparam}}}

\newcommand{\stochastic}[1]{\mathsf{#1}}
\newcommand{\stsignal}{\stochastic{\signal}}
\newcommand{\stdata}{\stochastic{\data}}

\DeclareMathOperator{\Expect}{\mathbb{E}}

\newcommand{\signaltrue}{\signal_{\text{true}}}

\newcommand{\primal}{\ensuremath{f}}
\newcommand{\dual}{\ensuremath{h}}

\usepackage[squaren,thickqspace]{SIunits}
\addunit{\pixel}{pixel}
\addunit{\voxel}{voxel}
\addunit{\decibel}{dB}
\addunit{\byte}{B}
\addunit{\hounsfield}{HU}

\usepackage{paperacronyms}
\acuse{MRI,ADMM,GPU,CPU}   

\usepackage{algorithm}
\usepackage[noend]{algpseudocode}

\title{Learning to solve inverse problems \\ using Wasserstein loss}

\author{
  Jonas~Adler$^{1,2}$, \; Axel~Ringh$^{1}$, \; Ozan \"Oktem$^{1}$, \; Johan~Karlsson$^{1}$
  \\
  $^1$Department of Mathematics,
  KTH Royal Institute of Technology \\
  $^2$Elekta Instrument AB \\ 
  \{ \texttt{jonasadl}, \texttt{aringh}, \texttt{ozan} \} \texttt{@kth.se}, \; \texttt{johan.karlsson@math.kth.se} \\
  \\
}

\begin{document}

\maketitle

\begin{abstract}
  We propose using the Wasserstein loss for training in inverse problems. In particular, we consider a learned primal-dual reconstruction scheme for ill-posed inverse problems using the Wasserstein distance as loss function in the learning. This is motivated by miss-alignments in training data, which when using standard mean squared error loss could severely degrade reconstruction quality. We prove that training with the Wasserstein loss gives a reconstruction operator that correctly compensates for miss-alignments in certain cases, whereas training with the mean squared error gives a smeared reconstruction. Moreover, we demonstrate these effects by training a reconstruction algorithm using both mean squared error and optimal transport loss for a problem in computerized tomography.
\end{abstract}

\section{Introduction}
In inverse problems the goal is to determine model parameters from indirect noisy observations.  Example of such problems arise in many different fields in science and engineering, e.g., in X-ray \ac{CT} \cite{natterer2001mathematical},
electron tomography \cite{oktem2015mathematics},
and magnetic resonance imaging \cite{brown2014magnetic}. 
Machine learning has recently also been applied in this area, especially in imaging applications. 
Using supervised machine-learning to solve inverse problems in imaging requires training data where ground truth images are paired with corresponding noisy indirect observations. The learning provides a mapping that associates observations to corresponding images. However, in several applications there are difficulties in obtaining the ground truth, e.g., in many cases it may have undergone a distortion. For example, a recent study showed that MRI images may be distorted by up to 4 mm due to, e.g., inhomogeneities in the main magnetic field \cite{Walker2014}. If these images are used for training, the learned MRI reconstruction will suffer in quality. Similar geometric inaccuracies arise in several other imaging modalities, such as Cone Beam CT and full waveform inversion in seismic imaging.

This work seeks to provide a scheme for learning a reconstruction scheme for an ill-posed inverse problem with a Wasserstein loss by leveraging upon recent advances in efficient solutions of optimal transport \cite{cuturi2013sinkhorn, karlsson2016generalized} and learned iterative schemes for inverse problems \cite{adler2017learned}. The proposed method is demonstrated on a computed tomography example, where we show a significant improvement compared to training the same network using mean squared error loss. In particular, using the Wasserstein loss instead of standard mean squared error gives a result that is more robust against potential miss-alignment in training data. 

\section{Background}\label{sec:background}

\subsection{Inverse problems}\label{subsec:inverseprob}
In inverse problems the goal is to reconstruct an estimate of the signal $\truesignal \in \RecSpace$ from noisy indirect measurements (data) $\data \in \DataSpace$ assuming
\begin{equation}\label{eq:InvProb}
\data = \ForwardOp(\truesignal) + \noise.  
\end{equation}  
In the above $\RecSpace$ and $\DataSpace$ are referred to as the reconstruction and data space, respectively. Both are typically Hilbert or Banach spaces. Moreover $\ForwardOp \colon \RecSpace \to \DataSpace$ denotes the forward operator, which models how a given signal gives rise to data in absence of noise. Finally, $\noise \in \DataSpace$ is the noise component of data.
Many inverse problems of interest are ill-posed, meaning that there is no uniques solution to \eqref{eq:InvProb} and hence there is no inverse to $\ForwardOp$. Typically reconstructions of $\truesignal$ are sensitive to the data and small errors gets amplified. One way to mitigate these effects is to use regularization \cite{engl2000regularization}. 

\paragraph{Variational regularization} 
In variational regularization one formulates the reconstruction problem as an optimization problem. To this end, one introduces a data discrepancy functional $f \mapsto \LogLikelihood(\ForwardOp(\signal), \data)$, where $\LogLikelihood: \DataSpace \times \DataSpace \to \Real$, that quantifies the miss-fit in data space, and a regularization functional $\Regularizer \colon \RecSpace \to \Real$  that encodes a priori information about $\truesignal$ by penalizing undesirable solutions. 
For a given $\data \in \DataSpace$, this gives an optimization problem of the form
\begin{equation}\label{eq:OptProb}
	\min_{\signal \in \RecSpace} \LogLikelihood(\ForwardOp(\signal), \data) + \lambda \Regularizer(\signal).
\end{equation}
Here, $\lambda$ acts as a trade-off parameter between the data discrepancy and regularization functional.
In many cases 
$\LogLikelihood$ is taken to be the negative data log-likelihood, e.g.,  $\LogLikelihood(\ForwardOp(\signal), \data)=\|\ForwardOp(\signal)-\data\|_2^2$ in the case of additive white Gaussian noise.
Moreover, a typical choice for regularization functional is total variation (TV) regularization,  $\Regularizer(\signal) =  \lVert \nabla \signal \rVert_1$ \cite{rudin1992nonlinear}.
These regularizers typically give rise to large scale and non-differentiable optimization problems, which requires advanced optimization algorithms.


\paragraph{Learning for inverse problems} 
In many applications, and so also in inverse problems, data driven approaches have shown dramatic improvements over the state-of-the-art \cite{LeBeHi15}. Using supervised learning to solve an inverse problem amounts to finding a parametrized operator $\ForwardOpInvLearned \colon \DataSpace \to \RecSpace$ where the parameters $\vparam$ are selected so that
\[
	\data = \ForwardOp(\truesignal) + \noise \implies \ForwardOpInvLearned(g) \approx \signaltrue.
\]
For inverse problems in image processing, such as denoising and deblurring, we have $\DataSpace = \RecSpace$ and it is possible to apply a wide range of widely studied machine learning techniques, such as fully convolutional deep neural networks with various architectures, including fully convolutional networks \cite{NIPS2008_3506} and denoising auto-encoders \cite{NIPS2012_4686}.

However, in more complicated inverse problems as in tomography, the data and reconstruction spaces are very different, e.g., their dimension after discretization may differ. For this reason, learning a mapping from $\DataSpace$ to $\RecSpace$ becomes nontrivial, and classical architectures that map, e.g., images to images using convolutional networks cannot be applied as-is. One solution is to use fully-connected layers as in \cite{PaGiKaLoMa04} for very small scale tomographic reconstruction problems. A major disadvantage with such a fully learned approach is that the parameters space has to be very high dimensional in order to be able to learn both the prior and the data model, which often renders it infeasible due to training time and lack of training data.

A more successful approach is to first apply some crude reconstruction operator $\ForwardOpInvP \colon \DataSpace \to \RecSpace$ and then use machine learning to post process the result. 
This separates the learning from the complications of mapping between spaces since the operator $\ForwardOpInvP$ can be applied off-line, prior to training. Such an approach has been demonstrated for tomographic reconstruction in \cite{PeBa13, WuGhChMa16}. Its drawback for ill posed inverse problems is that information is typically lost by using $\ForwardOpInvP$, and this information cannot be recovered by post processing. 

Finally, another approach is to incorporate the forward operator $\ForwardOp$ and its adjoint $\ForwardOp^*$ into the neural network. In these learned iterative schemes, classical neural networks are interlaced with applications of the forward and backward operator, thus allowing for the learned reconstruction operator to work directly from data without having to learn the data model. 
For example, in \cite{YaSuLiXu16}  an \ac{ADMM}-like scheme for Fourier inversion is learned and \cite{PuWe17} consider solving inverse problems typically arising in image restoration by a learned gradient-descent scheme. In \cite{adler2017solving} this later approach is shown to be applicable to large scale tomographic inversion. Finally, in \cite{adler2017learned} they apply learning in both spaces $\RecSpace$ and $\DataSpace$, yielding a Learned Primal-Dual scheme, and show that it outperforms learned post-processing for reconstruction of medical CT images.

\paragraph{Loss functions for learning}
Once the $\vparam$ parametrization of $\smash{\ForwardOpInvLearned}$ is set, the parameters are typically chosen by minimization of some loss functional $\loss$. Without doubt, the most common loss function is the mean squared error, also called $\mathcal{L}_2$ loss, given by
\begin{equation}\label{eq:L2Loss}
	\loss(\Theta) = \Expect_{\stsignal,\stdata} \Big[ \| \ForwardOpInvLearned(\stdata) - \stsignal \|_2^2 \Big].
\end{equation}
It has however been noted that it is sub-optimal for imaging, and a range of other loss functions have been investigated. These include the classical $\ell_p$ norms and the structural similarity index (SSIM) \cite{Zhao2017}, as well as more complex losses such as perceptual losses \cite{Johnson2016} and adversarial networks \cite{DBLP:journals/corr/MardaniGCVZATHD17}. 

Recently, optimal mass transport has also 
been considered as loss function for classification \cite{frogner2015learning} and generative models \cite{arjovsky2017wasserstein}. In this work we consider using optimal transport for training a reconstruction scheme for ill-posed inverse problems.

\subsection{Optimal mass transport and Sinkhorn iterations}\label{subsec:omt_sinkhorn}
In optimal mass transport the aim is to transform one distribution into another by moving the mass in a way that minimizes the cost of the movement.
For an introduction and overview of the topic, see, e.g., the monograph \cite{villani2008optimal}. Lately, the area has attracted a lot of research \cite{cuturi2013sinkhorn, cuturi2015asmoothed, chizat2015unbalanced} with applications to, e.g., signal processing \cite{haker2004optimal, georgiou2009metrics, jiang2012geometric, engquist2014application} and inverse problems \cite{benamou2015iterative, karlsson2016generalized}.

The optimal mass transport problem can be formulated as follows: let $\domain \subset \mR^d$ be a compact set, and let $\marginalone$ and $\marginaltwo$ be two measures, defined on $\domain$, with the same total mass.
Given a cost $c : \domain \times \domain \to \mR_+$ that describes the cost for transporting a unit mass from one point to another, find a (mass preserving) transference plan $M$ that is as cheap as possible. Here, the transference plan characterizes how to move the mass of $\marginalone$ in order to deform it into $\marginaltwo$. Letting the transference plan be a nonnegative measure $dM$ on the space $\domain\times \domain$ yields a linear programming problem in the space of measures:
\begin{align}\label{eq:Kantorovich}
T(\marginalone, \marginaltwo) =\quad
\min_{dM\ge 0}\quad&
\int_{(x_0,x_1)\in \domain\times \domain}c(x_0,x_1)dM(x_0,x_1)\\
\mbox{subject to }\; &\marginalone(x_0)dx_0=\int_{x_1\in \domain}dM(x_0,x_1),\nonumber\\
&\marginaltwo(x_1)dx_1=\int_{x_0\in \domain}dM(x_0,x_1).\nonumber
\end{align}
Although this formulation is only defined for measures $\marginalone$ and $\marginaltwo$ with the same total mass, it can also be extended to handle measures with unbalanced masses \cite{georgiou2009metrics, chizat2015unbalanced}. Moreover, 
under suitable conditions one can define the Wasserstein metrics $W_p$ using $T$, by taking $c(x_0,x_1)=d(x_0,x_1)^p$ for $p \ge 1$ and where $d$ is a metric on $\Omega$, and $W_p(\marginalone, \marginaltwo) := T(\marginalone, \marginaltwo)^{1/p}$  \cite[Definition 6.1]{villani2008optimal}. As the name indicates, $W_p$ is a metric on the set of nonnegative measures on $\domain$ with fixed mass \cite[Theorem 6.9]{villani2008optimal}, and $T$ is weak$^*$ continuous on this set.
One important property is that $T$ (and thus also $W_p$) does not only compare objects point by point, as standard $L^p$ metrics, but instead quantifies how the mass is moved. This makes optimal transport natural for quantifying uncertainty and modelling deformations \cite{jiang2012geometric,karlsson2013uncertainty}.

One way to solve the optimal transport problem in applications is to discretize $\domain$ and solve the corresponding finite-dimensional linear programming problem.
In this setting the two measures are represented by point masses on the discretization grid, i.e., by two vectors $\marginalone, \marginaltwo\in \mR_+^{n}$ where the element $[\mu_k]_i$ corresponds to the mass in the point $x_{(i)}\in \domain$ for $i=1,\ldots, n$ and $k=0,1$.
Moreover, a transference plan is represented by a matrix $M \in \mR_+^{n \times n}$ where the value $m_{ij}:=[M]_{ij}$ denotes the amount of mass transported from point $x_{(i)}$ to $x_{(j)}$. The associated cost of a transference plan is $\sum_{i,j=1}^{n}c_{ij}m_{ij}=\tr(C^TM)$, where $[C]_{ij}=c_{ij}=c(x_{(i)},x_{(j)})$ is the transportation cost from $x_{(i)}$ to $x_{(j)}$, and by discretizing the constraints we get that $M$ is a feasible transference plan from 
$\marginalone$ to $\marginaltwo$ if the row sums of $M$ is $\marginalone$ and the column sums of $M$ is $\marginaltwo$.
 The discrete version of \eqref{eq:Kantorovich} thus takes the form
\begin{align}\label{eq:T}
T(\marginalone, \marginaltwo) = 
\quad \min_{M\ge 0}\qquad  & \tr(C^T M)\\
 \mbox{subject to}\quad &\marginalone=M \ett_{n}
 \quad 
 \marginaltwo=M^T \ett_{n}, \nonumber
\end{align}
where $M \geq 0$ denotes element-wise non-negativity of the matrix. However, even though \eqref{eq:T} is a linear programming problem it is in many cases computationally infeasible due to the vast number of variables. Since $M \in \mR_+^{n \times n}$ the number of variables is $n^2$, and thus if one seek to solve the optimal transport problem between two $512\times 512$ images this results in more than $6 \cdot 10^{10}$ variables.

One approach for addressing this problem was proposed by Cuturi \cite{cuturi2013sinkhorn} that introduces an entropic regularizing term $D(M)=\sum_{i,j=1}^{n} (m_{ij}\log(m_{ij})-m_{ij}+1)$ for approximating the transference plan, so the resulting perturbed optimal transport problem reads as
\begin{align}\label{eq:Teps}
\quad \min_{M\ge 0} \qquad &  \tr(C^T M)+ \varepsilon D(M)\\
\mbox{subject to} \quad  & \marginalone=M \ett_{n} 
\quad 
\marginaltwo=M^T \ett_{n} \nonumber.
\end{align}
One can show that an optimal solution to \eqref{eq:Teps} is of the form
\begin{equation}\label{eq:structure}
M=\diag(\iterateone)K \diag(\iteratetwo),
\end{equation}
where $K=\exp(-C/\varepsilon )$ (point-wise exponential) is known, and $\iterateone, \iteratetwo\in \mR_+^{n}$ are unknown. This shows that the solution is parameterized by only $2n$ variables. Moreover, the two vectors can be computed iteratively by so called Sinkhorn iterations, i.e., alternatingly compute $\iterateone$ and $\iteratetwo$ that matches $\marginalone$ and $\marginaltwo$ respectively. This is summarizied in Algorithm~\ref{alg:sinkhorn} where $\odot$  denotes elementwise multiplication and $./$ elementwise division.
The procedure has been shown to have a linear convergence rate, see \cite{cuturi2013sinkhorn} and references therein. 

Moreover, when the underlying cost $c(x_0,x_1)$ is translation invariant the discretized cost matrix $C$, and thus also the transformation $K$, gets a Toeplitz-block-Toeplitz structure. This structure can be used in order to compute $K\iteratetwo$ and $K^T\iterateone$ efficiently using the fast Fourier transform in $\ordo(n \log n)$, instead of naive matrix-vector multiplication in $\ordo(n^2)$ \cite{karlsson2016generalized}. This is crucial for applications in imaging since for images of size $512 \times 512$ pixels one would have to explicitly store and multiply with matrices of size $262144 \times 262144$ .
	\begin{algorithm}[H]
		\caption{Sinkhorn iterations for computing entropy-regularized optimal transport \cite{cuturi2013sinkhorn}}\label{alg:sinkhorn}
		\begin{algorithmic}[1]		  
		  \State {\bf Input} $C, \varepsilon, \marginalone, \marginaltwo$
			\State initialize $\iteratetwo_0 > 0$ and $K = \exp(-C/\varepsilon )$
			\For{$i = 1, \dots, N$}
			\State $\iterateone_i \gets \marginalone./(K\iteratetwo_{i-1})$
			\State $\iteratetwo_i \gets \marginaltwo./(K^T\iterateone_{i})$
			\EndFor
			\State \Return $\iterateone_N^T (K \odot C) \iteratetwo_N$
		\end{algorithmic}
	\end{algorithm}

\section{Learning a reconstruction operator using \mbox{Wasserstein} loss} 
\label{sec:learningwithomt}

In this work we propose to use entropy regularized optimal transport \eqref{eq:Teps} to train a reconstruction operator, i.e., to select the parameters as
\begin{equation}\label{eq:optTransLoss}
	\Theta^* \in \arg\min_\Theta \Expect_{\stsignal,\stdata} \Big[ T( \ForwardOpInvLearned(\stdata), \stsignal ) \Big].
\end{equation}
This should give better results when data $g$ is not aligned with the ground truth $\signal$. To see this, consider the case when $f$ is a point mass. In that case training the network with the $\mathcal{L}_2$ loss \eqref{eq:L2Loss} will (in the ideal case) result in a perfect reconstruction composed with a convolution that ``smears'' the reconstruction over the area of possible miss-alignment. On the other hand since optimal mass transport does not only compare objects point-wise, the network will (in the ideal case) learn a perfect reconstruction combined with a movement of the object to the corresponding barycenter (centroid) of the miss-alignment. These statements are made more precise in the following propositions. Formal definitions and proofs are deferred to the appendix.

\begin{proposition}\label{prop:l2_smoothin}
Let $g \in \mathcal{L}_2(\mR^n)$, let $\tau$ be a $\mR^n$-valued random variable with probability measure $dP(t)$, and let $g_\tau(x) := g(x - \tau)$. Then there exists a function $f \in \mathcal{L}_2(\mR^n)$ that minimizes $\Expect_\tau[\| f - g_\tau \|_2^2]$, and this $f$ has the form
\[
f(x) = (dP * g)(x) := \int_{\mR^n} g(x - t) dP(t).
\]
\end{proposition}

\begin{proposition}\label{prop:omt_barrycenter}
Let $\delta(x)$ be the Dirac delta function on $\mR^n$, let $\tau$ be a $\mR^n$-valued random variable with probability measure $dP(t)$, and let $\delta_\tau(x) := \delta(x - \tau)$. Then
\[
\Expect_\tau[ T(\delta_\tau, \mu) ]=\int_{\mR^n} \Big( \underbrace{\int_{\mR^n} c(t,x) dP(t)}_{:=F(x)} \Big) d\mu(x) \quad \mbox{ for } \mu\ge 0  \mbox{ and } \int_{\mR^n} d\mu(x)=1  
\]
and $\Expect_\tau[ T(\delta_\tau, \mu) ]=\infty$ otherwise. Furthermore, finding a $\mu$ that minimizes $\Expect_\tau[ T(\delta_\tau, \mu) ]$ is equivalent to finding the global minimizers to $F(x)$.
In particular, if (i) the probability measure $dP$ is symmetric around its mean, (ii) the underlying cost $c$ is of the form $c(t,x) = d(x - t)$, where $d$ is convex and symmetric, and (iii) $dP$ and $d$ are such that
\[
\int_{\mR^n} d(x - t) dP(t) \qquad \text{and} \qquad \int_{\mR^n} \partial d(x - t) dP(t)
\]
both exist and are finite for all $x \in \mR^n$, then $\mu(x) = \delta(x - \Expect [\tau])$ is an optimal solution. 
Furthermore, if $d$ is also strictly convex, then this is the unique minimizer.
\end{proposition}

To illustrate Propositions~\ref{prop:l2_smoothin} and ~\ref{prop:omt_barrycenter} we consider the following example.


\begin{example}\label{ex:uniform}
	Let $\tau$ be uniformly distributed on $[-1,1]$, and let $c(x_0, x_1) = (x_0 - x_1)^2$. This gives
	\[
	F(x) = \frac{1}{2} \int_{-1}^1 (x - t)^2 dt = \frac{1}{3} + x^2,
	\]
	which has minimum $x=0$, and hence the (unique) minimizer to  $\Expect_\tau[ T(\delta_\tau, \mu) ]$ is $\mu(x) = \delta(x)$. For the $\mathcal{L}_2$ case with the uniform distribution, the minimizer of $\Expect_\tau[\| f - g_\tau \|_2^2]$ is the smoothed function $g \ast \frac{1}{2}\chi_{[-1, 1]}$.
\end{example}


The most common choice of distance $c$ is to use the squared norm $c(x_0, x_1) = \|x_0 - x_1\|^2$, as in the previous example. In this case the result of Proposition~\ref{prop:omt_barrycenter} can be strengthened, as shown in the following example.

\begin{example}\label{ex:discr}
Let $\tau$ be a $\mR^n$-valued random variable with probability measure $dP(t)$ with finite first and second moments, and let $c(x_0, x_1) = \|x_0 - x_1\|^2$. This gives
\[
F(x) = \int_{\mR^n} (x - t)^2 dP(t) = x^2 - 2x \Expect[\tau] + \Expect[\tau^2],
\]
which has a unique global minimum in $x = \Expect[\tau]$ and hence $\mu(x) = \delta(x - \Expect[\tau])$.
\end{example}

\section{Implementation and evaluation}\label{sec:Eval}
We use the recently proposed learned primal-dual structure in \cite{adler2017learned} for learning a reconstruction operator $\smash{\ForwardOpInvLearned}$ for solving the inverse problem in \eqref{eq:InvProb}. 
In this algorithm, a sequence of small blocks work alternatingly in the data (dual) space $\DataSpace$ and the reconstruction (primal) space $\RecSpace$ and are connected using the forward operator $\ForwardOp$ and its adjoint $\ForwardOp^*$. The algorithm works with any differentiable operator $\ForwardOp$, but we state the version for linear operators for simplicity in \cref{alg:learned_pd}.

\begin{algorithm}[H]
	\caption{Learned Primal-Dual reconstruction algorithm}\label{alg:learned_pd}
	\begin{algorithmic}[1]
		\State Initialize $\primal_0 \in \RecSpace^{N_{\text{primal}}}, \dual_0 \in \ProdSpace^{N_{\text{dual}}}$
		\For{$i = 1, \dots, I$}
		\State $\dual_i \gets
		\Gamma_{\vparam_i^d}\bigl(\dual_{i - 1}, \ForwardOp(\primal_{i-1}^{(2)}), \data\bigr)$
		\State $\primal_i \gets 
		\Lambda_{\vparam_i^p}\bigl(\primal_{i - 1}, \ForwardOp^*(\dual_i^{(1)}) \bigr)$
		\EndFor
		\State $\ForwardOpInvLearned(\data) \mathrel{\mathop:}= \primal_I^{(1)}$
	\end{algorithmic}
\end{algorithm}

\begin{figure}[t]
	\centering	
	\includegraphics[width=0.95\linewidth]{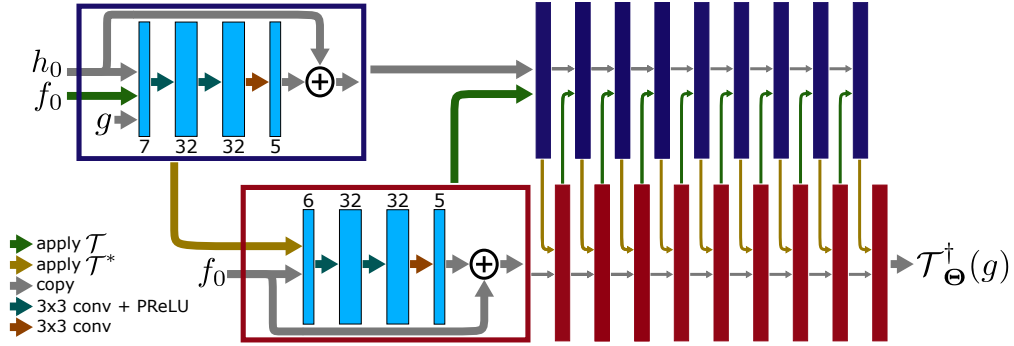}
	\caption{Network architecture used to solve the inverse problem. Dual and primal iterates are in blue and red boxes, respectively. Several arrows pointing to the same box indicates concatenation. The initial values $\primal_0, \dual_0$ enter from the left, while the data $\data$ is supplied to the dual iterates.}
	\label{fig:network_graph}
\end{figure}

The method was implemented using ODL \cite{adler2017ODL}, ASTRA \cite{vanaerle2016fast}, and TensorFlow \cite{abadi2016tensorflow}. We used the reference implementation%
\footnote{\url{https://github.com/adler-j/learned_primal_dual}} with default parameters, i.e., the number of blocks in the primal and dual space was $I=10$, and the number of primal and dual variables was set to $N_{\text{primal}} = N_{\text{dual}} = 5$. Moreover, the blocks used a residual structure and had three layers of $3 \times 3$ convolutions with $32$ filters. PReLU nonlinearities were used. Thus, this corresponds to a residual CNN with convolutional depth of $10 \cdot 2 \cdot 3 = 60$, as shown in graphical format in \cref{fig:network_graph}. We used zero initial values, $\primal_0 = \dual_0 = 0$. 

We compare a learned reconstruction operator of this form when trained using $\mathcal{L}_2$ loss \eqref{eq:L2Loss} and using optimal transport loss \eqref{eq:optTransLoss}. Moreover, the evaluation is done on a problem similar to the evaluation problem in \cite{adler2017learned, adler2017solving}, i.e., on a problem in computed tomography. More specifically, training is done on data that consists of randomly generated  circles on a domain of \unit{512 \times 512}{\pixel}, and the forward operator $\ForwardOp$ is the ray transform \cite{natterer2001mathematical}. What makes this an ill-posed problem is that the data acquisition is done from only 30 views with 727 parallel lines. Moreover, the source of noise is two-fold in this set-up:
(i) the pairs $(g_i, f_i)$ of data sets and phantoms are not aligned, meaning that the data is computed from a phantom with a random change in position. This random change is independent for the different circles, and for each circle it is a shift which is uniformly distributed over $[-40, 40]$ pixels, both in up-down and left-right direction.
(ii) on the data computed from the shifted phantom, 5\% additive Gaussian noise was added. For an example, see \cref{fig:example_data}.

\begin{figure}[t]
	\centering	
	\begin{subfigure}[t]{.28\linewidth}	
		\includegraphics[width=\linewidth, trim={23mm 17mm 32mm 6mm}, clip]{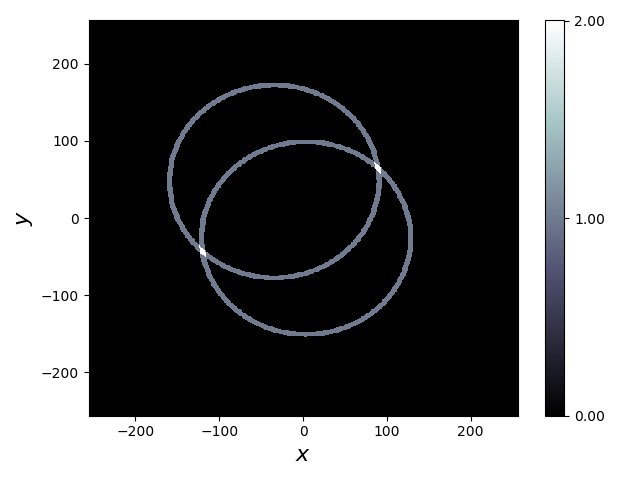}
		\caption{Phantom}
		\label{subfig:phantom}
	\end{subfigure}
	\begin{subfigure}[t]{.06\linewidth}
		\vspace{-50pt}		
		{\huge$\xrightarrow{}$}
	\end{subfigure}
	\begin{subfigure}[t]{.28\linewidth}	
		\includegraphics[width=\linewidth, trim={23mm 17mm 32mm 6mm}, clip]{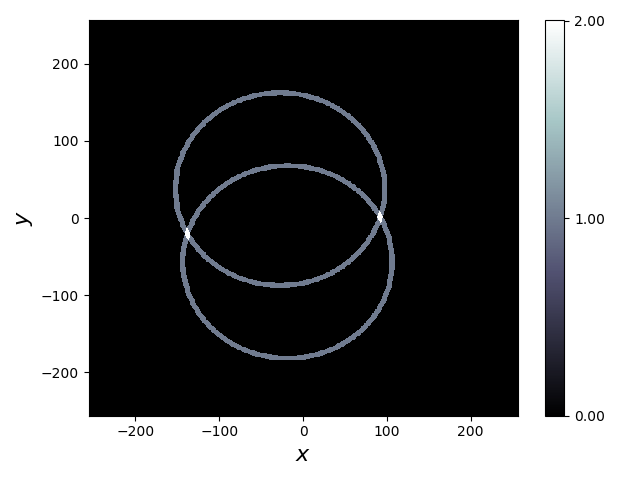}
		\caption{Translated phantom}
		\label{subfig:translated_phantom}
	\end{subfigure}
	\begin{subfigure}[t]{.06\linewidth}
		\vspace{-50pt}		
		{\huge$\xrightarrow{}$}
	\end{subfigure}
	\begin{subfigure}[t]{.28\linewidth}
		\includegraphics[width=\textwidth, trim={23mm 17mm 32mm 6mm}, clip]{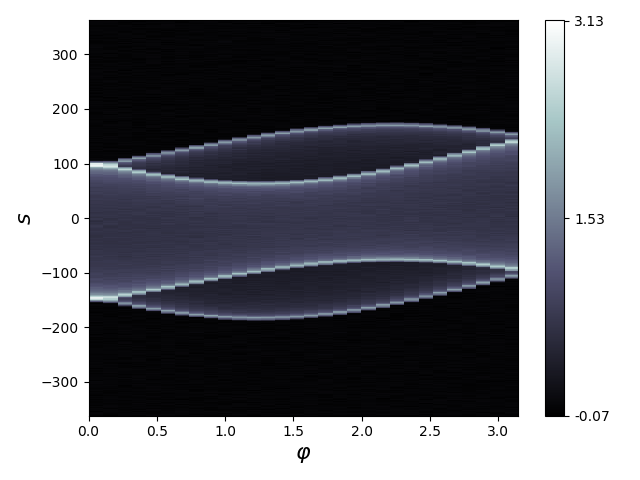}
		\caption{Data}
		\label{subfig:noisy_sinogram}
	\end{subfigure}
	\caption{Example of data generation process used for training and validation, where \ref{subfig:phantom} shows an example phantom, \ref{subfig:translated_phantom} is the phantom with a random translation and \ref{subfig:noisy_sinogram} is the data (sinogram) corresponding to \ref{subfig:translated_phantom} with additive white noise on top. The pair $(g_i, f_i)$ = (\ref{subfig:noisy_sinogram}, \ref{subfig:phantom}) is what is used in the training.} 
	\label{fig:example_data}
\end{figure}

The optimal mass transport distance computed with Sinkhorn iterations was used as loss function, where we used the transport cost
\[
	c(x_1, x_2) = \left( 1 - e^{-\|x_1 - x_2\|^4 / 80^4} \right).
\]
This was chosen since it heavily penalizes large movements, while not diverging to infinity which causes numerical instabilities. Moreover, $c(x_1, x_2)^{1/4}$ is in fact a metric on $\Real^2$ (see \cref{lem:metric} in the appendix) and thus $W_4(\marginalone, \marginaltwo) := T(\marginalone, \marginaltwo)^{1/4}$ gives rise to a Wasserstein metric on the space of images, where $T(\marginalone, \marginaltwo)$ is the optimal mass transport distance with the transport cost $c(x_1, x_2)$.
Since this cost is translation invariant, the matrix-vector multiplications $K\iterateone$ and $K^T\iteratetwo$ can be done with fast Fourier transform, as mentioned in \cref{subsec:omt_sinkhorn}, and this was implemented in Tensorflow. We used 10 Sinkhorn iterations with  entropy regularization $\varepsilon = 10^{-3}$, to approximate the optimal mass transport. Automatic differentiation was used to back-propagate the result during training. 

Since the optimal mass transport function \eqref{eq:Teps} is only finite for marginals $\marginalone$ and $\marginaltwo$ with the same total mass, in the training we normalize the output of the reconstruction $\smash{\ForwardOpInvLearned(g)}$ with ${\rm mass}(f) / {\rm mass}(\smash{\ForwardOpInvLearned(g)})$. This makes $\smash{\ForwardOpInvLearned(g)}$ invariant with respect to the total mass, which is undesirable. To compensate for this, a small penalization on the error in total mass was added to the loss function.

The training also followed \cite{adler2017learned} closely. In particular, we used $2 \cdot 10^4$ batches of size $1$, using the ADAM optimizer \cite{kingma2014adam} with default values except for $\beta_2 = 0.99$. The learning rate (step length) used was cosine annealing \cite{loshchilov2016sgdr} with initial step length $10^{-3}$. Moreover, in order to improve training stability we performed gradient norm clipping \cite{pascanu2012understanding} with norms limited to 1. The convolution parameters were initialized using Xavier initialization \cite{glorot2010understanding}, and all biases were initialized to zero. The training took approximately 3 hours using a single Titan X GPU. The source code used to replicate these experiments are available online \footnote{\url{https://github.com/adler-j/wasserstein_inverse_problems}}

Results are presented in \cref{fig:results}. As can be seen, the reconstruction using $\mathcal{L}_2$ loss ``smears'' the reconstruction to an extent where the shape is impossible to recover. On the other hand, the reconstruction using the Wasserstein loss retains the over-all global shape of the object, although relative and exact positions of the circles are not recovered.

\begin{figure}[t]
	\centering	
	\begin{subfigure}[t]{.24\linewidth}	
		\includegraphics[width=\linewidth, trim={23mm 17mm 32mm 6mm}, clip]{figures/512/truth.png}
		\caption{Phantom}
		\label{subfig:phantom2}
	\end{subfigure}
	\begin{subfigure}[t]{.24\linewidth}	
		\includegraphics[width=\linewidth, trim={23mm 17mm 32mm 6mm}, clip]{figures/512/translated_phantom.png}
		\caption{Translated phantom}
		\label{subfig:translated_phantom2}
	\end{subfigure}
	\begin{subfigure}[t]{.24\linewidth}	
		\includegraphics[width=\linewidth, trim={23mm 17mm 32mm 6mm}, clip]{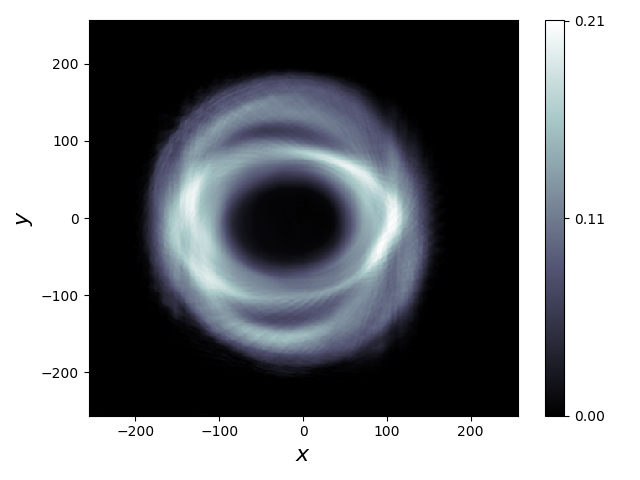}
		\caption{Mean squared error loss}
		\label{subfig:L2Loss}
	\end{subfigure}
	\begin{subfigure}[t]{.24\linewidth}
		\includegraphics[width=\textwidth, trim={23mm 17mm 32mm 6mm}, clip]{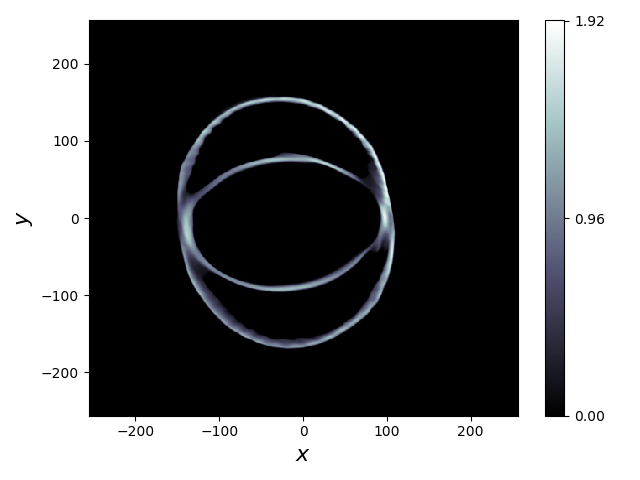}
		\caption{Optimal transport loss}
		\label{subfig:omtLoss}
	\end{subfigure}
	\caption{ In \ref{subfig:phantom2} we show the validation phantom, which was generated from the same training set but not used in training, in \ref{subfig:translated_phantom2} the translated phantom from which the validation data was computed, in \ref{subfig:L2Loss} a reconstruction with neural network trained using mean squared error loss \eqref{eq:L2Loss}, and in \ref{subfig:omtLoss} 	a reconstruction with neural network trained using optimal mass transport loss \eqref{eq:optTransLoss}.
	}
	\label{fig:results}
\end{figure}

\section{Conclusions and future work}
In this work we have considered using Wasserstein loss to train a neural network for solving ill-posed inverse problems in imaging where data is not aligned with the ground truth. We give a theoretical motivation for why this should give better results compared to standard mean squared error loss, and demonstrate it on a problem in computed tomography. In the future, we hope that this method can be applied to other inverse problems and to other problems in imaging such as segmentation.

%

\appendix
\section*{Appendix: Deferred definition and proofs}

\begin{proof}[Proof of Proposition~\ref{prop:l2_smoothin}]
To show that $f(x) = (dP*g)(x) \in \mathcal{L}_2(\mR^n)$ minimizes $\Expect_\tau[ \| f - g_\tau\|_2^2 ]$ we expand the expression and use Fubini's theorem to get
\[
\Expect_\tau[ \| f - g_\tau \|_2^2 ] = \int_{\mR^n} \left( \int_{\mR^n} |f(x) - g_t(x)|^2 dx \right) dP(t) = \int_{\mR^n} \left( \int_{\mR^n} (f(x) - g_t(x))^2 dP(t) \right) dx.
\]
Rearranging terms and using that $\int_{\mR^n} dP(t) dt = 1$, this can be written as
\[
\Expect_\tau[ \| f - g_\tau \|_2^2 ] =
\int_{\mR^n} \left( f(x) - \int_{\mR^n} g_t(x) dP(t) \right)^2 dx + c,
\]
where $c$ is a constant. Using this it follows that the minimizing $f$ is of the form
\[
f(x) = \int_{\mR^n} g_t(x) dP(t).
\]
To see that $f \in \mathcal{L}_2(\mR^n)$ we note that, by using Fubini's theorem, we have
\begin{align*}
\| f \|_ 2^2 &  = \int_{\mR^n} \left( \int_{\mR^n} g_t(x) dP(t) \right)^2 dx = \int_{\mR^n} \int_{\mR^n} \left( \int_{\mR^n} g_s(x) g_t(x) dx \right) dP(s) dP(t) \\
& \leq \int_{\mR^n} \int_{\mR^n} \left( \frac{1}{2} \int_{\mR^n} g_s(x)^2 + g_t(x)^2 dx \right) dP(s) dP(t) = \| g \|_2^2 < \infty
\end{align*}
where the first inequality is the arithmetic-geometric mean inequality. This completes the proof.
\end{proof}

\begin{definition}
Let $h : \mR^n \to \mR$. A subgradient to $h$ in a point $y$ is a vector $v$ so that
\[
h(x) \geq h(y) + \langle v, x - y \rangle, \qquad \forall \, x \in \mR^n. 
\]
The set of all subgradients in a point $y$ is called the subdifferential of $h$ at $y$, and is denoted by $\partial h (y)$. This is a set-valued operator, and for any measure $d\nu$ on $\mR^n$ we define 
$(d\nu * \partial h)(y) := \int_{\mR^n} \partial h(t) d\nu(y -t)$
to be the set-valued operator
\[
y \mapsto \left\{ \int_{\mR^n}  v(t) d\nu(y - t) \in \mR^n \mid v(t) \in \partial h(t) \right\}.
\]
\end{definition}

\begin{proof}[Proof of Proposition~\ref{prop:omt_barrycenter}]
We consider finding the marginal $\mu$ that minimize $\Expect_\tau[ T(\delta_\tau, \mu) ]$. Without loss of generality we assume that $\tau$ is zero-mean, since otherwise we simply consider $\tau - \Expect[\tau]$ which is a zero-mean random variable. First we note that $T(\delta_t, \mu)$ is only finite for nonegative measures $\mu$ with total mass $1$, and hence $\Expect_\tau[ T(\delta_\tau, \mu) ]$ is only finite for such measures. Second, for such a $\mu$ we have
\[
T(\delta_t, \mu) = \int_{\mR^n} c(t,x) d\mu(x),
\]
since one needs to transport all mass in $\mu$ into the point $t$ where $\delta_t$ has its mass. Using this and expanding the expression for the expectation gives that
\[
\Expect_\tau[ T(\delta_\tau, \mu) ] = \int_{\mR^n} T(\delta_t, \mu) dP(t) = \int_{\mR^n} \left( \int_{\mR^n} c(t,x) dP(t) \right) d\mu(x),
\]
where we have used Fubini's theorem in the last step. This completes the first half of the statement.

To prove the second half of the statement, note that the optimal $\mu$ have support only in the global minimas of the function 
\[
F(x) := \int_{\mR^n} d(x - t) dP(t) = \int_{\mR^n}  d(t) dP(x - t),
\]
which by assumption exists and is finite. Now, since $d$ is convex we have that
\begin{equation}\label{eq:d_subgrad}
d(x) \geq d(y) + \langle \partial d(y), x - y \rangle, \quad  \text{for all } x,y \in \mR^n,
\end{equation}
and convolving this inequality with $dP$ gives the inequality
\begin{equation}\label{eq:F_subgrad}
F(x) \geq F(y) + \langle \int_{\mR^n} \partial d(y - t) dP(t) , x - y \rangle, \quad \text{for all } \, x,y \in \mR^n,
\end{equation}
where all terms exist and are bounded by assumption. This shows that 
\[
\int_{\mR^n} \partial d(y - t) dP(t) \subset \partial F(y).
\]
Now, since $d$ is symmetric we have that $\partial d$ is anti-symmetric, i.e., that $\partial d(x) = - \partial d (-x)$, since 
\[
d(x) = d(-x) \geq d(-y) + \langle \partial d(-y), -x + y \rangle = d(y) + \langle -\partial d(-y), x - y \rangle.
\]
Therefore
\[
\partial F(0) \supset  \int_{\mR^n}  \partial d(-t) dP(t) \ni 0,
\]
where the last inclusion follows since $dP$ is symmetric and $\partial d$ is anti-symmetric. Now, since $0 \in \partial F(0)$ we have that $x=0$ is a global minimizer to $F(x)$ \cite[Theorem 16.2]{bauschke2011convex}, and thus one optimal solution to the problem is $\mu(x) = \delta(x)$. Now, if $d$ is strictly convex, the inequality \eqref{eq:d_subgrad} is strict for $x \neq y$, and thus \eqref{eq:F_subgrad} is also strict, which shows that the optimal solution is unique.
\end{proof}

\begin{lemma}\label{lem:metric}
Let $\| \cdot \|$ be a norm on $\Real^m$. Then
\[
	d(x_1, x_2) = \bigl( 1 - e^{-\|x_1 - x_2\|^n} \bigr)^{\frac{1}{n}}.
\]
is a metric on $\Real^m$ for $n \geq 1$.
\end{lemma}

\begin{proof}
It is easily seen that $d(x_1, x_2)$ is symmetric, nonnegative, and equal to zero if only if $x_1 = x_2$. Thus we only need to verify that the triangle inequality holds. To this end we note that if
		\begin{equation}\label{eq:wts}
		\bigl(1 - e^{-(a + b)^n} \bigr)^{\frac{1}{n}} \leq \bigl(1 - e^{-a^n}\bigr)^{\frac{1}{n}} + \bigl(1 - e^{- b^n}\bigr)^{\frac{1}{n}}, \text{ for all } a,b \geq 0,
		\end{equation}
		for all $n \geq 1$, then by taking $a = \| x_1 - x_2 \|$, $b = \| x_2 - x_3 \|$, and using the triangle inequality for the norm $\| \cdot \|$ we have that
		\begin{align*}
		d(x_1, x_3) & = \bigl( 1 - e^{-\| x_1 - x_3 \|^n} \bigr)^{\frac{1}{n}} \leq \bigl(1 - e^{-(\| x_1 - x_2 \| + \| x_2 - x_3 \|)^n}\bigr)^{\frac{1}{n}} \\
		& \leq \bigl(1 - e^{-\| x_1 - x_2 \|^n} \bigr)^{\frac{1}{n}} + \bigl(1 - e^{- \| x_2 - x_3 \|^n}\bigr)^{\frac{1}{n}} = d(x_1, x_2) + d(x_2, x_3).
		\end{align*}
    Therefore we will show that \eqref{eq:wts} holds for all $n \geq 1$, and to do so we will
    \begin{enumerate}
    \item[(i)] show that if a function $g : \mathbb{R}_+ \to \mathbb{R}_+$ fulfills $g(0) = 0$, $g(x)' \geq 0$, $g''(x) \leq 0$ for all $x \in \mathbb{R}_+$, then $g(x_1 + x_2) \leq g(x_1) + g(x_2)$,
    \item[(ii)] show that for $x \geq 0$ the map $x \mapsto (1 - e^{-x^n})^{\frac{1}{n}}$ fulfills the assumptions in (i) for any $n \geq 1$.
    \end{enumerate}     
    To show (i) we note that
    \begin{align*}
    g(x_1 + x_2) & = \int_0^{x_1+x_2} g'(t)dt = \int_0^{x_1} g'(t)dt + \int_{x_1}^{x_1+x_2} g'(t)dt \\
    & \leq \int_0^{x_1} g'(t)dt + \int_0^{x_2} g'(t)dt = g(x_1) + g(x_2),    
    \end{align*}
    where the inequality uses that $g'(t) \geq g'(x + t)$ for any $x,t \geq 0$ since $g''(x) \leq 0$ for all $x \geq 0$.

    To show (ii), let $g(x) := (1 - e^{-x^n})^{\frac{1}{n}}$ and observe that $g(0) = 0$. Differentiating $g$ twice gives
    \begin{align*}
    & g'(x) = \frac{e^{-x^n} (1 - e^{-x^n})^{\frac{1}{n}} x^{n - 1}}{1 - e^{-x^n}}\\
    & g''(x) = - \frac{(1-e^{-x^n})^{\frac{1}{n}} x^{n-2} \big( \overbrace{ne^{x^n}x^n - x^n + e^{x^n} - ne^{x^n} + n -1 }^{=: h_n(x^n)} \big)}{(e^{x^n} - 1)^2}.
    \end{align*}    	
    For $x \geq 0$ we see that $g'(x) \geq 0$ for all $n \geq 1$. Moreover, for $x \geq 0$ we see that $g''(x) \leq 0$ for all $x \geq 0$ and for all $n \geq 1$ if and only if $h_n(x^n) \geq 0$. With the change of variable $x^n = y$, we thus want to show that $h_n(y) \geq 0$ for all $y \geq 0$ and all $n \geq 1$. To see this we note that $h_n(0) = 0$ and that
    \[
    h_n'(y) = ne^y y + e^y - 1 \geq 0 \; \text{ for all } y \geq 0 \text{ and } n\geq 1.
    \]
    This shows (ii), and hence completes the proof.
\end{proof}

\bibliographystyle{plain}
\small
\bibliography{refs}

\end{document}